\documentclass[11pt]{amsart}
\usepackage{amsmath}
\usepackage{amssymb}
\usepackage{mathpazo}
\usepackage{eucal}
\usepackage{amsthm}
\usepackage{amscd}
\usepackage{hyperref}
\usepackage{url}
\usepackage{skull}
\usepackage{adjustbox}

\usepackage{tikz}
\usetikzlibrary{shapes.geometric, arrows}
\usetikzlibrary{shapes,arrows}
\usepackage[latin1]{inputenc}
\usepackage{tikz}
\usetikzlibrary{shapes,arrows}
\usetikzlibrary{shapes.multipart}
\usepackage{verbatim}
\usepackage{tkz-berge}
\usepackage{tikz-qtree}

\usepackage{amssymb,}
\usepackage[]{amsmath, amsthm, amsfonts,graphicx, amscd,}
\usepackage[all,cmtip]{xy}
\input amssym.def \input amssym
\usepackage{forest}
\usepackage{mdframed}
\usepackage{framed}

\DeclareMathOperator{\Ldim}{Ldim}

\DeclareMathOperator{\dom}{dom}
\DeclareMathOperator{\SC}{SC}
\DeclareMathOperator{\C}{C}
\DeclareMathOperator{\LC}{LC}

\newtheorem{thm}{Theorem}[section]
\newtheorem{cor}[thm]{Corollary}

\newtheorem {fact}[thm]{Fact}

\newtheorem{prop}[thm]{Proposition}

\newtheorem {lem}[thm]{Lemma}

\theoremstyle{remark}

\newtheorem{np*}{Non-Proof}

\theoremstyle{definition}
\newtheorem{defn}[thm]{Definition}

\newtheorem{exam}[thm]{Example}

\def\Ind{\setbox0=\hbox{$x$}\kern\wd0\hbox to 0pt{\hss$\mid$\hss} \lower.9\ht0\hbox to 0pt{\hss$\smile$\hss}\kern\wd0}
\def\Notind{\setbox0=\hbox{$x$}\kern\wd0\hbox to 0pt{\mathchardef \nn=12854\hss$\nn$\kern1.4\wd0\hss}\hbox to 0pt{\hss$\mid$\hss}\lower.9\ht0 \hbox to 0pt{\hss$\smile$\hss}\kern\wd0}

\numberwithin{equation}{section}

\newcommand{\m}{\mathbb }
\newcommand{\mc}{\mathcal }

\title{Bounds in Query Learning}
\author{Hunter Chase}
\address{Department of Mathematics, UIC, Chicago IL}
\email{hchase2@uic.edu}

\author{James Freitag}
\address{Department of Mathematics, UIC, Chicago IL}
\email{freitagj@gmail.com}

\thanks{Partially supported by NSF grant no. 1700095}

\begin{document}

\maketitle

\begin{abstract}%
We introduce new combinatorial quantities for concept classes, and prove lower and upper bounds for learning complexity in several models of query learning in terms of various combinatorial quantities. Our approach is flexible and powerful enough to enough to give new and very short proofs of the efficient learnability of several prominent examples (e.g. regular languages and regular $\omega$-languages), in some cases also producing new bounds on the number of queries. In the setting of equivalence plus membership queries, we give an algorithm which learns a class in polynomially many queries whenever any such algorithm exists. 

We also study equivalence query learning in a randomized model, producing new bounds on the expected number of queries required to learn an arbitrary concept. Many of the techniques and notions of dimension draw inspiration from or are related to notions from model theory, and these connections are explained. We also use techniques from query learning to mildly improve a result of Laskowski regarding compression schemes. 

\end{abstract}


\section{Introduction}

Fix a set $X$ and denote by $\mc P (X)$ the collection of all subsets of $X$. A \emph{concept class}\footnote{We will also sometimes call $\mc C$ a set system on $X$} $\mathcal{C}$ on $X$ is a subset of $ \mc P (X)$. In the equivalence query (EQ) learning model, a learner attempts to identify a target set $A \in \mathcal{C}$ by means of a series of data requests called \emph{equivalence queries}. 
The learner has full knowledge of $\mathcal{C}$, as well as a hypothesis class $\mc H$ with $\mc C \subseteq \mathcal{H} \subseteq \mc P(X).$
An \emph{equivalence query} consists of the learner submitting a hypothesis $B \in \mathcal{H}$ to a teacher, who either returns \emph{yes} if $A = B$, or a counterexample $x \in A \triangle B$. In the former case, the learner has learned $A$, and in the latter case, the learner uses the new information to update and submit a new hypothesis. In sections \ref{basicEQ} and \ref{applications}, the teacher may be assumed to be adversarial and the worst case number of queries required to learn any concept is analyzed. In section \ref{randomeq}, we consider the case in which the teacher selects counterexamples randomly according to a fixed but arbitrary distribution. 

We will also consider learning with equivalence and membership queries (EQ+MQ). In a membership query, a learner submits a single element $x$ from the base set $X$ to the teacher, who returns the value $A(x)$, where $A$ is the target concept. In this setting, the learner may choose to make either type of query at any stage, submitting any $x \in X$ for a membership query or submitting any $B \in \mathcal{H}$ for an equivalence query. The learner learns the target concept $A$ when they submit $A$ as an equivalence query.

With Theorems \ref{EQlearn} and \ref{EQMQalg}, we give upper bounds for the number of queries required for EQ and EQ+MQ learning a class $\mc C$ with hypotheses $\mc H$ in terms of the \emph{Littlestone dimension of $\mc C$}, denoted $\Ldim(\mathcal{C})$, and the \emph{consistency dimension of $\mc C$ with respect to $\mc H$}, denoted $\C(\mathcal{C}, \mathcal{H})$. We also give lower bounds for the number of required queries in terms of these quantities. In the EQ+MQ setting, the bounds are tight enough to completely characterize when a problem is efficiently learnable. Littlestone dimension is well-known in learning theory \cite{littlestone1988learning} and model theory.\footnote{In model theory, Littlestone dimension is called Shelah 2-rank, see \cite{MLMT} for additional details.} 

Consistency dimension and the related notion of strong consistency dimension are more subtle, which we detail in section \ref{basicEQ}. When $\mc H$ is taken to be $\mc P(X)$, $\C(\mc C, \mc H)=1$; for various examples of set systems with $\mc H = \mc C$, one has $\C(\mc C, \mc H) = \infty$. In \ref{obtainC}, we define a new invariant, the consistency threshold of $\mc C$, and provide a construction (for arbitrary $\mc C$) of a hypothesis class $\mc H$ which is not much more complicated than $\mc C$ (of the same Littlestone dimension as $\mc C$) such that $\C ( \mc C, \mc H) \leq \Ldim(\mc C)+1.$ In \ref{SCvsC}, we compare our bounds and invariants to those previously appearing in the literature. 

Theorems \ref{EQlearn} and \ref{EQMQalg} can be used to establish efficient learnability in specific applied settings \emph{if} one can obtain appropriate bounds on Littlestone dimension and consistency dimension. Let $(\mc C_n, \mc H_n)$ be a collection of concept and hypothesis classes which depends on some parameter $n$. Typically, we are thinking of finite classes which grow with $n$. We prove that whenever $\mc C_n$ can be learned by an algorithm using polynomially many membership queries and equivalence queries from $\mathcal{H}_n$, there must be polynomial bounds on Littlestone and consistency dimension. Moreover, whenever such an algorithm exists, the algorithm given in Theorem \ref{EQMQalg} accomplishes this. 

Finally, to close section \ref{basicEQ}, we explain the connection between strong consistency dimension and a model theoretic property called the finite cover property (fcp), or rather its negation, referred to henceforth as the nfcp. We show that if $\mc C$ is the set system given by uniform instances of a fixed first order formula $\phi$, and $\mc H$ is the collection of externally $\phi$-definable sets, then $(\mc C, \mc H)$ has finite strong consistency dimension if and only if $\phi$ has the nfcp. 

In section \ref{applications} we demonstrate the practicality of our approach by providing simple and fast proofs of the efficient learnability of regular languages and certain $\omega$-languages, reproving results of \cite{angluin1987learning,angluin2016learning,fisman2018families,fisman2018inferring}. Besides the conceptual simplicity of the approach, the bounds in learning complexity resulting from our algorithm have some novel aspects. For instance, our bounds have no dependence on the length of the strings provided to the learner as counterexamples, in contrast to existing algorithms. 


In section \ref{randomeq} we turn to a randomized variant of EQ-learning in which the teacher is required to choose counterexamples randomly from a known probability distribution on $X$. \cite{angluin2017power} show that for a concept class of size $n$, there is an algorithm in which the expected number of queries to learn any concept is at most $\log _2 (n).$ It is natural to wonder whether there is a notion of dimension which can be used to bound the expected number of queries. In fact, Angluin and Dohrn \cite[Theorem 25]{angluin2017power} already consider this, and show that the VC-dimension of the concept class is a lower bound on the number of expected queries. However, \cite[Theorem 26]{angluin2017power}, using an example of \cite{littlestone1988learning}, shows that the VC-dimension \emph{cannot} provide an upper bound for the number of queries. We show that the Littlestone dimension provides such an upper bound; we give an algorithm which yields a bound which is linear in the Littlestone dimension for the expected number of queries needed to learn any concept.

In section \ref{compression}, we introduce compression schemes for concept classes. Specifically, the notion we work with is equivalent to $d$-compression with $b$ extra bits (of Floyd and Warmuth \cite{floyd1995sample}). In \cite{johnson2010compression}, Laskowski and Johnson proved that the concept class corresponding to a stable formula has an extended $d$-compression for some $d$. Later, a result of Laskowski appearing as \cite[Theorem 4.1.3]{guingonanip} in fact showed that 
one could take $d$ equal to the Shelah 2-rank (Littlestone dimension) and uses $2^d$ many reconstruction functions. We show that $d+1$ many reconstruction functions suffice.


\section{A combinatorial characterization of EQ-learnability} \label{basicEQ} 
Often, one assumes that $X$ is finite, and the emphasis is placed on finding bounds on the number of queries it may take to learn any $A \in \mathcal{C}$. We also consider the case where $X$ is infinite, for which we give the following definition.

\begin{defn}
	Let $\mathcal{C}$ and $\mathcal{H}$ be set systems on a set $X$. $\mathcal{C}$ is \emph{learnable with equivalence queries} from $\mathcal{H}$ if there exists some $n < \omega$ and some algorithm to submit hypotheses from $\mathcal{H}$ such that any concept $A \in \mathcal{C}$ is learnable in at most $n$ equivalence queries, given any teacher returning counterexamples. Let $\LC^{EQ}(\mathcal{C}, \mathcal{H})$ be the least such $n$ if $\mathcal{C}$ is learnable with equivalence queries from $\mathcal{H}$, and $\LC^{EQ}(\mathcal{C}, \mathcal{H}) = \infty$ otherwise. 
	
	$\LC^{EQ}(\mathcal{C}, \mathcal{H})$ is called the \emph{learning complexity}, representing the optimal number of queries needed in the worst-case scenario. 
	
	Similarly, $\mathcal{C}$ is learnable with equivalence queries from $\mathcal{H}$ and membership queries if there exists some $n < \omega$ and some algorithm to submit membership queries from $X$ or equivalence queries from $\mathcal{H}$ such that any concept $A \in \mathcal{C}$ is learnable in at most $n$ equivalence queries. The learning complexity is defined similarly and is denoted by $\LC^{EQ + MQ}(\mathcal{C}, \mathcal{H})$.

\end{defn}



\subsection{EQ-learnability from Littlestone and consistency dimension}

\begin{prop} \cite[Theorems 5 and 6]{littlestone1988learning} \label{stableEQlearningeasy}
	If $\LC^{EQ}(\mathcal{C}, \mathcal{H}) \leq d+1$, then $\Ldim(\mathcal{C}) \leq d$. If $\mathcal{H} = \mathcal{P}(X)$, then the converse holds.
\end{prop}

\begin{proof}
Suppose $\Ldim(\mathcal{C}) \geq d + 1$. We show that we can force the learner to use at least $d+2$ equivalence queries. Construct a binary element tree of height $d+1$ with proper labels from $\mathcal{C}$ witnessing $\Ldim(\mathcal{C}) \geq d + 1$. Given the first hypothesis $H_0$ from the learner, return the element on the 0th level on the tree as a counterexample. Continue this, returning the element on the $i$th level along the path consistent with previous counterexamples as the counterexample to hypothesis $H_i$. We will return $d+1$ counterexamples, and the learner still requires one more hypothesis to identify the concept. Since this will occur for one of the proper labels $A$ of the binary element tree, we have forced the learner to use at least $d+2$ equivalence queries for some $A \in \mathcal{C}$.
	
	\bigskip
	
	Suppose $\Ldim(\mathcal{C}) = d < \infty$. Let $\mathcal{C}_0 = \mathcal{C}$. Inductively define $\mathcal{C}_i$, $i = 1, \ldots, d$ as follows. Given $\mathcal{C}_i$, for any $x \in X$ and $j \in \{0,1\}$, let
	\[
	\mathcal{C}_i^{(x, j)} := \{ A \in \mathcal{C}_i \, | \, \chi_A(x) = j \},
	\]
	where $\chi_A$ is the characteristic function on $A$. Let
	\[
	B_i := \{ x \in X \, | \, \Ldim(\mathcal{C}_i^{(x,1)}) \geq \Ldim(\mathcal{C}_i^{(x,0)}) \}.
	\]
	Submit $B_i$ as the hypothesis. If $B_i$ is correct, we are done. Otherwise, we receive a counterexample $x_i$. Set
	\[
	\mathcal{C}_{i+1} := \{ A \in V_i \, | \, \chi_A(x_i) \neq \chi_{B_i}(x_i) \}
	\]
	to be the concepts which have the correct label for $x_i$. Observe that at each stage, $\Ldim(\mathcal{C}_{i+1}) < \Ldim(\mathcal{C}_i)$. Therefore, if we make $d$ queries without correctly identifying the target, then we must have $\Ldim(\mathcal{C}_d) = 0$. Then $V_d$ is a singleton, which must be the target concept.	
\end{proof}

Notice in particular that if $\Ldim(\mathcal{C}) = \infty$, then $\mathcal{C}$ cannot be learned with equivalence queries, even with $\mathcal{H} = \mathcal{P}(X)$. The assumption that $\mathcal{H} = \mathcal{P}(X)$ makes learning straightforward, but this may be too strong for many settings. However, without some additional hypotheses on $\mathcal{H}$, learnability may already be hopeless, even for \emph{very simple} set systems. For instance, let $\mathcal{C}$ be the set of singletons. If $\mathcal{H} = \mathcal{C}$, then we may take as long as $|X|$ to learn if $X$ is finite, or never learn at all if $X$ is infinite. However, if the learner is allowed to guess $\emptyset$, this forces the teacher to identify the target singleton.


The strategy of Proposition \ref{stableEQlearningeasy} permeates both learnability and non-learnability proofs; identifying a specific set amounts to reducing the Littlestone dimension of the family of possible concepts to 0; actually submitting the target concept before the Littlestone dimension reaches 0 can be thought of as a best-case scenario that we cannot rely on. Non-learnability then amounts to an inability to reduce the Littlestone dimension of the family of possible concepts to 0 through a series of finitely many equivalence queries. The main purpose of this section is to give precise conditions on $\mathcal{H}$ and $\mathcal{C}$ which \emph{characterize} learnability.

\begin{defn}
	Given a set $X$, a \emph{partially specified subset} $A$ of $X$ is a partial function $A: X \rightarrow \{0, 1\}$.
	\begin{itemize}
		\item Say $x \in A$ if $A(x) = 1$, $x \notin A$ if $A(X) = 0$, and membership of $x$ is unspecified otherwise. The \emph{domain} of $A$, $\dom(A)$, is $A^{-1}(\{0,1\})$. Call $A$ \emph{total} if $\dom(A) = X$. We identify subsets $A \subseteq X$ with total partially specified subsets. The \emph{size} of $A$, $|A|$, is the cardinality of $\dom(A)$.
		\item Given two partially specified subsets $A$ and $B$, write $A \sqsubseteq B$ if $A$ and $B$ agree on $\dom(A)$; call $A$ a \emph{restriction} of $B$ and $B$ an \emph{extension} of $A$.
		\item Given a set $Y \subseteq \dom(A)$, the restriction $A|_Y$ of $A$ to $Y$ is the partial function where $A|_Y(x) = A(x)$ for all $x \in Y$, and is unspecified otherwise. 
		\item Given a set system $\mathcal{C}$ on $X$, $A$ is \emph{$n$-consistent} with $\mathcal{C}$ if every size $n$ restriction of $A$ has an extension in $\mathcal{C}$. Otherwise, say $A$ is \emph{$n$-inconsistent}. $A$ is \emph{finitely consistent} with $\mathcal{C}$ if every restriction of $A$ of finite size has an extension in $\mathcal{C}$---that is, $A$ is $n$-consistent with $\mathcal{C}$ for all $n < \omega$.		
	\end{itemize}
\end{defn}

The following definition is a translation into set systems of a definition that first appeared in \cite{balcazar2002consistencydimension}. 

\begin{defn}
	The \emph{consistency dimension} of $\mathcal{C}$ with respect to $\mathcal{H}$, denoted $\C(\mathcal{C}, \mathcal{H})$, is the least integer $n$ such that for every subset $A \subseteq X$ (viewed as a total partially specified subset), if $A$ is $n$-consistent with $\mathcal{C}$, then $A \in \mathcal{H}$. If no such $n$ exists, then say $\C(\mathcal{C}, \mathcal{H}) = \infty$.
\end{defn}

Observe that $\C(\mathcal{C}, \mathcal{H}) = 1$ iff $\mathcal{H}$ shatters\footnote{Recall that a set system $\mathcal{C}$ shatters a set $A$ if, for all $B \subseteq A$, there is $C \in \mathcal{C}$ such that $C \cap A = B$.} the set of all elements $x \in X$ such that there are $A_0$ and $A_1$ in $\mathcal{C}$ such that $x \notin A_0$ but $x \in A_1$. In this case, it is possible to learn any concept in $\mathcal{C}$ in at most $\Ldim(\mathcal{C}) + 1$ equivalence queries, using the method of Proposition \ref{stableEQlearningeasy}. So we may assume that $\C(\mathcal{C}, \mathcal{H}) > 1$.

\begin{lem}\label{unionlemma}
	Suppose that for each $i < n$, $\mathcal{C}_i$ is a concept class on $X$ and $\mathcal{H}_i$ is a hypothesis class on $X$. Suppose that $\LC^{EQ}(\mathcal{C}_i, \mathcal{H}_i) = m_i$. Then $\LC^{EQ}(\mathcal{C}, \mathcal{H}) \leq \sum_{i < n} m_i$, where $\mathcal{C} := \cup_{i < n} \mathcal{C}_i$ and $\mathcal{H} := \cup_{i < n} \mathcal{H}_i$.
\end{lem}

\begin{proof} 
We give the proof for $n = 2$; then the result for $n > 2$ follows easily by induction. 

To learn a target concept $A \in \mathcal{C} = \mathcal{C}_0 \cup \mathcal{C}_1$ with hypotheses from $\mathcal{H} = \mathcal{H}_0 \cup \mathcal{H}_1$, begin by assuming that $A \in \mathcal{C}_0$. Attempt to learn $A$ by making guesses from $\mathcal{H}_0$, according to the procedure by which any concept in $\mathcal{C}_0$ is learnable in at most $m_0$ many queries. If, after making $m_0$ many queries, we have failed to learn $A$, then we conclude that $A \notin \mathcal{C}_0$, whence $A \in \mathcal{C}_1$. We can then learn $A$ in at most $m_1$ many additional queries with guesses from $\mathcal{H}_1$.
\end{proof}

We can now give an upper bound for the learning complexity in terms of Littlestone dimension and consistency dimension. 

\begin{thm}\label{EQlearn}
	Suppose $\Ldim(\mathcal{C}) = d < \infty$ and $1 < \C(\mathcal{C}, \mathcal{H}) = c < \infty$. Then $\LC^{EQ}(\mathcal{C}, \mathcal{H}) \leq c^d$.
\end{thm}

\begin{proof} 
We proceed by induction on $d$. The base case, $d = 0$, is trivial, as then $\mathcal{C}$ is a singleton. 

Suppose there is some element $x$ such that $\Ldim(\mathcal{C} \cap x) < d + 1$ and $\Ldim(\mathcal{C} \setminus x) < d + 1$, where $\mathcal{C} \cap x := \{ A \in \mathcal{C} \, | \, x \in A \}$ and $\mathcal{C} \setminus x := \{ A \in \mathcal{C} \, | \, x \notin A \}$. Then by induction, any concept in $\mathcal{C} \cap x$ can be learned in at most $c^d$ queries with guesses from $\mathcal{H}$, and the same is true for $\mathcal{C} \setminus x$. Then by Lemma \ref{unionlemma}, any concept in $\mathcal{C}$ can be learned in at most $2c^d \leq c^{d+1}$ equivalence queries.

If no such $x$ exists, then for all $x$, either $\Ldim(\mathcal{C} \cap x) = d + 1$ or $\Ldim(\mathcal{C} \setminus x) = d + 1$. Let $B$ be such that $x \in B$ iff $\Ldim(\mathcal{C} \cap x) = d+1$.

If $B \in \mathcal{H}$, then we submit $B$ as our query. If we are incorrect, then by choice of $B$, the class $\mathcal{C}'$ of concepts consistent with the counterexample $x_0$ will have Littlestone dimension $\leq d$. By induction, any concept in $\mathcal{C}'$ can be learned in at most $c^d$ many queries, and so we learn $a$ in at most $c^d + 1 \leq c^{d+1}$ queries.

If $B \notin H$, then, since $\C(\mathcal{C}, \mathcal{H}) = c$, there are some $x_0, \ldots, x_{c-1}$ such that there is no $A \in \mathcal{C}$ such that $B|_{\{x_0, \ldots, x_{c-1}\}} \sqsubseteq A$. Then, with notation as in the proof of Proposition \ref{stableEQlearningeasy},
\[
\mathcal{C} = (\mathcal{C}^{(x_0, 1-B(x_0))}) \cup \ldots \cup (\mathcal{C}^{(x_{c-1}, 1-B(x_{c-1}))}),
\]
and $\Ldim(\mathcal{C}^{(x_i, 1-B(x_i))}) \leq d$ for each $i$. Then, by induction, for each $i$, any concept in $\mathcal{C}^{(x_i, 1-B(x_i))}$ can be learned in at most $c^d$ many queries with guesses from $\mathcal{H}$. By Lemma \ref{unionlemma}, any concept in $\mathcal{C}$ can be learned in at most $c^{d+1}$ many queries with guesses from $\mathcal{H}$. 
\end{proof}

On the other hand, Proposition \ref{stableEQlearningeasy} gives a lower bound of $\Ldim(\mathcal{C}) + 1 \leq \LC^{EQ}(\mathcal{C}, \mathcal{H})$. There is also a lower bound for learning complexity in terms of consistency dimension: 

\begin{prop}\cite[Theorem 2]{balcazar2002consistencydimension}\label{prop:Clowerbound}
	Suppose there is some partially specified subset $A$ which is $n$-consistent with $\mathcal{C}$ but which does not have a total extension in $\mathcal{H}$. Then $n < \LC^{EQ}(\mathcal{C}, \mathcal{H})$.
\end{prop}

\begin{proof}
By hypothesis, given any equivalence query $H$, the teacher can find some $x \in \dom(A)$ such that $H(x) \neq A(x)$. Moreover, since $A$ is $n$-consistent with $\mathcal{C}$, the teacher is able to return a counterexample of this form for the first $n$ equivalence queries. Thus $\mathcal{C}$ cannot be learned with fewer than $n + 1$ equivalence queries from $\mathcal{H}$.
\end{proof}

In particular, if $\C(\mathcal{C}, \mathcal{H}) \geq c$, then there is some subset $A$ which is $(c-1)$-consistent with $\mathcal{C}$ but which does not belong to $\mathcal{H}$. Then $c \leq \LC^{EQ}(\mathcal{C}, \mathcal{H})$. So $\C(\mathcal{C}, \mathcal{H}) \leq \LC^{EQ}(\mathcal{C}, \mathcal{H})$. In fact, we will obtain a stronger bound using strong consistency dimension in section \ref{SCvsC}. 

Furthermore, if $\C(\mathcal{C}, \mathcal{H}) = \infty$, then $\mathcal{C}$ cannot be learned with equivalence queries from $\mathcal{H}$. Combining Theorem \ref{EQlearn} and Propositions \ref{stableEQlearningeasy} and \ref{prop:Clowerbound}, we obtain the following:

\begin{thm}\label{EQcategorization}
	$\mathcal{C}$ is learnable with equivalence queries from $\mathcal{H}$ iff $\Ldim(\mathcal{C}) < \infty$ and $\C(\mathcal{C}, \mathcal{H}) < \infty$.
\end{thm}


\subsection{Obtaining finite consistency dimension}  \label{obtainC}
We have established that finite consistency dimension is essential for EQ-learning. The central question we answer in this subsection is: given $\mc C$, can one obtain a hypothesis class $\mc H$ which is not much more complicated than $\mc C$ with the property that $\C(\mc C, \mc H)$ is finite?


\begin{defn}
	Fix a set system $\mathcal{C}$ on a set $X$. $\mathcal{C}$ has \emph{consistency threshold} $n < \infty$ if, given any hypothesis class $\mathcal{H} \supset \mathcal{C}$, we have that
	\[
	\C(\mathcal{C}, \mathcal{H}) < \infty \quad \text{iff} \quad \C(\mathcal{C}, \mathcal{H}) \leq n.
	\]
\end{defn}

\begin{lem} \label{lem:consistentextension}
	Suppose $A$ is a partially specified subset finitely consistent with $\mathcal{C}$. Then there is a total extension $A' \sqsupseteq A$ finitely consistent with $\mathcal{C}$. 
\end{lem}

\begin{proof}

Let $X = \{x_\alpha \, | \, \alpha < |X| \}$ be a well-ordering of $X$. Let $A_0 = A$. We inductively define a $\sqsubseteq$-chain of partially specified subsets $A_\alpha$, where each $A_\alpha$ is defined on $\dom(A) \cup \{x_\xi \, | \, \xi < \alpha \}$ and is finitely consistent with $\mathcal{C}$. For $\alpha$ a limit ordinal, set $A_\alpha = \cup_{\xi < \alpha} A_\xi$. It is clear that $A_\alpha$ is finitely consistent with $\mathcal{C}$ if all $A_\xi$ for $\xi < \alpha$ are. 

At any successor stage $\alpha + 1$, if $x_\alpha \in \dom(A_\alpha)$, set $A_{\alpha + 1} = A_\alpha$. Otherwise, we must extend $A_\alpha$ to $x_\alpha$ while remaining finitely consistent with $\mathcal{C}$. Assume for contradiction that neither $B_0 := A_\alpha \cup \{x_\alpha \mapsto 0 \}$ nor $B_1 := A_\alpha \cup \{x_\alpha \mapsto 1 \}$ are finitely consistent with $\mathcal{C}$. Then there are finite sets $Y_0, Y_1 \subseteq \dom(A_\alpha)$ such that $B_0 |_{Y_0 \cup \{a_\alpha\}}$ and $B_1 |_{Y_1 \cup \{a_\alpha \}}$ have no extension in $\mathcal{C}$. But $A_\alpha |_{Y_0 \cup Y_1}$ has an extension $B$ in $\mathcal{C}$, and $B$ must be an extension of either $B_0 |_{Y_0 \cup \{a_\alpha\}}$ or $B_1 |_{Y_1 \cup \{a_\alpha\}}$, a contradiction. So $A_\alpha$ has a finitely consistent extension to $x_\alpha$, and we set $A_{\alpha + 1}$ to be such an extension.

We then take $A' = \cup_{\xi < |X|} A_\xi$.
\end{proof}

\begin{prop}\label{prop:needfincon}
	Let $\mathcal{C}, \mc H$ be a set systems and let $A$ be a partially specified subset. 
	The following are equivalent: 
	\begin{description}
		\item[(i)] $A$ is finitely consistent with $\mathcal{C}$.
		\item[(ii)] If $\C(\mathcal{C}, \mathcal{H}) < \infty$, then there is a total extension $A' \sqsupseteq A$ in $\mathcal{H}$.
	\end{description}
\end{prop}
\begin{proof} 
(i) $\Rightarrow$ (ii): Let $A' \sqsupseteq A$ be a total extension finitely consistent with $\mathcal{C}$. If $\C(\mathcal{C}, \mathcal{H}) < \infty$, then $A' \in \mathcal{H}$.

(ii) $\Rightarrow$ (i): We show the contrapositive. Suppose that $A$ is not finitely consistent with $\mathcal{C}$, witnessed by some size $n$ restriction $A_0$, which is a $\sqsubseteq$-minimal such restriction. We find some $\mathcal{H}$ such that $\C(\mathcal{C}, \mathcal{H}) < \infty$ but $\mathcal{H}$ contains no total extension of $A$. Let $\mathcal{H}$ be the collection of all (total partially specified) subsets which are not extensions of $A_0$. So $A$ has no total extension in $\mathcal{H}$. We claim that $\C(\mathcal{C}, \mathcal{H}) \leq n$. Indeed, observe that given any (total partially specified) subset $B$ that is $n$-consistent with $\mathcal{C}$, we have $A_0 \not \sqsubseteq B$, and then $B \in \mathcal{H}$.

\end{proof}

In particular, if $\C(\mathcal{C}, \mathcal{H}) < \infty$, then $\mathcal{H}$ contains all finitely consistent subsets. That is, extensions of all finitely consistent partially specified subsets (equivalently, by Lemma \ref{lem:consistentextension}, all finitely consistent total partially specified subsets) are necessary to obtain $\C(\mathcal{C}, \mathcal{H}) < \infty$. Consistency threshold classifies when this is a sufficient condition.

\begin{prop} \label{prop:SCproperties}
The following are equivalent:
	\begin{description}
		\item[(i)] $\mathcal{C}$ has consistency threshold $\leq n < \infty$.
		\item[(ii)] For all (total partially specified) subsets $A$, if $A$ is $n$-consistent with $\mathcal{C}$, then $A$ is finitely consistent with $\mathcal{C}$.
		\item[(iii)] If $\mathcal{H}$ contains all finitely consistent (total partially specified) subsets, then $\C(\mathcal{C}, \mathcal{H}) \leq n$. 
		
	\end{description}
\end{prop}
\begin{proof}
	
(i) $\Rightarrow$ (ii): Assume for contradiction that there is some total $A$ which is $n$-consistent but not finitely consistent. Let $m$ be minimal such that $A$ is $m$-inconsistent. Then there is a size $m$ restriction $A' \sqsubseteq A$ that has no extension in $\mathcal{C}$. Then let $\mathcal{H}$ contain all subsets which do not extend $A'$.
		
We claim that $\C(\mathcal{C}, \mathcal{H}) = m$. Note that $A$ witnesses that $\C(\mathcal{C}, \mathcal{H}) \geq m$. On the other hand, observe that given any partially specified subset $B$ that is $m$-consistent with $\mathcal{C}$, we have $A' \not \sqsubseteq B$, and then it is easy to see that $B$ has a total extension in $\mathcal{H}$.	
		
(ii) $\Rightarrow$ (iii): If $\mathcal{H}$ contains all finitely consistent subsets, and all $n$-consistent subsets are finitely consistent, then $\C(\mathcal{C}, \mathcal{H}) \leq n$ holds immediately.
		
(iii) $\Rightarrow$ (i): By Proposition \ref{prop:needfincon}, if $\C(\mathcal{C}, \mathcal{H}) < \infty$, then $\mathcal{H}$ already has all finitely consistent subsets. Then $\C(\mathcal{C}, \mathcal{H}) \leq n$.
\end{proof}

In particular, if $\mathcal{C}$ has finite consistency threshold, then $\C(\mathcal{C}, \mathcal{H}) < \infty$ iff $\mathcal{H}$ contains all finitely consistent subsets.

\begin{cor} \label{cor:largennothreshold}
	Suppose $\mathcal{C}$ does not have finite consistency threshold. Then for arbitrarily large $n$, there is some total subset $A_n$ which is $n$-consistent but not $(n+1)$-consistent with $\mathcal{C}$.
\end{cor}



Finite consistency threshold is not strictly necessary to provide a positive answer to the central question of this subsection; nevertheless, it does identify a clear qualitative dividing line. When $\mathcal{C}$ has finite consistency threshold, $\mathcal{H}$ only needs to contain all finitely consistent subsets; letting $\mathcal{H}_\infty$ be the set of all finitely consistent subsets, we obtain a minimum hypothesis class such that learning is possible. 

Where $\mathcal{C}$ does not have finite consistency threshold, more is required; we must add some hypotheses which are inconsistent with the concepts in $\mathcal{C}$, and there is no minimal $\mathcal{H}$ such that learning is possible. However, for each $m$, we can replace ``finitely consistent'' with ``$m$-consistent'' to obtain a class $\mathcal{H}_m$ such that $\C(\mathcal{C}, \mathcal{H}_m) \leq m$---let $\mathcal{H}_m$ be the collection of all subsets which are $m$-consistent with $\mathcal{C}$. Note that $\mathcal{H}_m$ is clearly the minimum hypothesis class such that $\C(\mathcal{C}, \mathcal{H}) \leq m$.

Note that for all $m$, $\mathcal{H}_\infty \subseteq \mathcal{H}_m$. By Proposition \ref{prop:SCproperties}, if $\mathcal{C}$ has consistency threshold $n$, then for all $m \geq n$, $\mathcal{H}_m = \mathcal{H}_n = \mathcal{H}_\infty$. If $\mathcal{C}$ does not have finite consistency threshold, there is no minimal $\mathcal{H}$ such that $\C(\mathcal{C}, \mathcal{H}) < \infty$; by Corollary \ref{cor:largennothreshold}, if $\C(\mathcal{C}, \mathcal{H}) = m$, then there is $m' \geq m$ such that $\mathcal{H}_{m'} \subsetneq \mathcal{H}$.


By choosing $m$ appropriately, given any $\mathcal{C}$, we can find a hypothesis class such that $\C(\mathcal{C}, \mathcal{H}) < \infty$ without increasing the Littlestone dimension; that is, $\Ldim(\mathcal{H}) = \Ldim(\mathcal{C})$.

\begin{thm}\label{thm:fixedLdim}
	Suppose $\Ldim(\mathcal{C}) = d < \infty$. Then there is $\mathcal{H}$ such that $\C(\mathcal{C}, \mathcal{H}) < \infty$ and $\Ldim(\mathcal{H}) = \Ldim(\mathcal{C})$. 
	Furthermore, we can find such an $\mathcal{H}$ such that $\C(\mathcal{C}, \mathcal{H}) \leq \Ldim(\mathcal{C}) + 1$.
\end{thm}

\begin{proof} 

Fix some $m > d = \Ldim(\mathcal{C})$. Let $\mathcal{H}_m$ be the collection of all subsets which are $m$-consistent with $\mathcal{C}$. It is immediate that $\C(\mathcal{C}, \mathcal{H}_m) \leq m < \infty$.
	
Assume for contradiction that $\Ldim(\mathcal{H}_m) > \Ldim(\mathcal{C})$. Consider a binary element tree of height $d+1$ that can be properly labeled with elements of $\mathcal{H}_m$; in particular, there is some leaf which cannot be labeled with an element of $\mathcal{C}$. Consider such a leaf. The path through the binary element tree to this leaf defines a partially specified subset $A$ that is $(d+1)$-inconsistent with $\mathcal{C}$. In particular, any total extension is $(d+1)$-inconsistent, so $m$-inconsistent, and so does not belong to $\mathcal{H}_m$. This contradicts our ability to label the leaf with an element of $\mathcal{H}_m$.
	
In particular, recall that when $\mathcal{C}$ has finite consistency threshold $n$, $A$ is $n$-consistent with $\mathcal{C}$ iff it is finitely consistent with $\mathcal{C}$. So setting $\mathcal{H}_m$ as above with $m$ at least the finite consistency threshold amounts to setting $\mathcal{H}_m$ to be the collection of all finitely consistent partially specified subsets. In this case, $\Ldim(\mathcal{H}_m) = \Ldim(\mathcal{C})$ even if $m \leq d$, as increasing the Littlestone dimension requires adding something inconsistent with $\mathcal{C}$.

Regardless of whether $\mathcal{C}$ has finite consistency dimension, we can let $m = d+1$. Then $\C(\mathcal{C}, \mathcal{H}_m) \leq m = d+1$.
\end{proof}

\subsection{From consistency to strong consistency}\label{SCvsC}

From an algorithms perspective, the result of Theorem \ref{EQlearn} is unsatisfactory, since it is exponential in $\Ldim(\mathcal{C})$. We give an example to show that, without modification, we cannot expect a significant improvement.

\begin{exam}\label{example:exponential}
	Fix $c > 2$ and $d$. Let $\{a_\tau \, | \, \tau \in [c]^i,\, 1 \leq i \leq d \}$ be distinct elements indexed by finite nonempty sequences of length at most $d$ from $[c]$. For $\sigma \in [c]^d$, let $B_\sigma = \{a_\tau \, | \, \tau \subseteq \sigma \}$. Let $\mathcal{C} = \{B_\sigma \, | \, \sigma \in [c]^d \}$. Then $\Ldim(\mathcal{C}) = d$. 
	
	If we take $\mathcal{C}$ to also be our hypothesis class, then $\C(\mathcal{C}, \mathcal{C}) = c+1$. Indeed, the (total partially specified) subset $A = \{a_0\}$ is $c$-consistent but not $(c+1)$ consistent with $\mathcal{C}$, witnessed by the restriction of $A$ to $\{a_0, a_{0,0}, \ldots, a_{0,c-1}\}$, so $\C(\mathcal{C}, \mathcal{C}) \geq c+1$. On the other hand, if $A$ is a subset $(c+1)$-consistent with $\mathcal{C}$, then, by induction on the length of $\tau$, for each $1 \leq i \leq d$, $A$ contains exactly one $a_\tau$ with $\tau = i$, so $A \in \mathcal{C}$.
	
	However, it may take as long as $c^d$ many equivalence queries to learn; if the teacher returns $a_\sigma$ as a counterexample to hypothesis $A_\sigma$, then the learner can only eliminate $A_\sigma$.
\end{exam}

The most promising modification is the following variant of consistency dimension, which also appeared in \cite{balcazar2002consistencydimension} in a slightly different form.

\begin{defn}
	The \emph{strong consistency dimension} of $\mathcal{C}$ with respect to $\mathcal{H}$, denoted $\SC(\mathcal{C}, \mathcal{H})$, is the least integer $n$ such that for every partially specified subset $A$, if $A$ is $n$-consistent with $\mathcal{C}$, then $A$ has an extension in $\mathcal{H}$. If no such $n$ exists, then say $\SC(\mathcal{C}, \mathcal{H}) = \infty$.
\end{defn}

We therefore make the stronger requirement that all partially specified subsets that are $n$-consistent be consistent, rather than just all totally partially specified subsets. It is immediate from the definition that $\C(\mathcal{C}, \mathcal{H}) \leq \SC(\mathcal{C}, \mathcal{H})$. At the smallest levels, consistency dimension and strong consistency dimension are equal.

\begin{prop} \label{prop:C-SC12equal}
	If $\C(\mathcal{C}, \mathcal{H}) = 1$, then $\SC(\mathcal{C}, \mathcal{H}) = 1$. If $\C(\mathcal{C}, \mathcal{H}) = 2$, then $\SC(\mathcal{C}, \mathcal{H}) = 2$.
\end{prop}
\begin{proof} 
Observe that $\C(\mathcal{C}, \mathcal{H}) = 1$ iff $\SC(\mathcal{C}, \mathcal{H}) = 1$ iff $\mathcal{H}$ shatters the set of all elements $x \in X$ such that there are $A_0$ and $A_1$ in $\mathcal{C}$ such that $x \notin A_0$ but $x \in A_1$.
	
Suppose that $\C(\mathcal{C}, \mathcal{H}) = 2$. Let $A$ be a partially specified subset that is 2-consistent with $\mathcal{C}$. We wish to find a total extension of $A$ in $\mathcal{H}$. It suffices to find a total extension $B \sqsupseteq A$ that is 2-consistent with $\mathcal{C}$.
	
Let $X = \{x_\alpha \, | \, \alpha < |X| \}$ be a well-ordering of $X$. Let $A_0 = A$. We inductively define a $\sqsubseteq$-chain of partially specified subsets $A_\alpha$, where each $A_\alpha$ is defined on $\dom(A) \cup \{x_\xi \, | \, \xi < \alpha \}$ and is 2-consistent with $\mathcal{C}$. For $\alpha$ a limit ordinal, set $A_\alpha = \cup_{\xi < \alpha} A_\xi$. It is clear that $A_\alpha$ is 2-consistent with $\mathcal{C}$ if all $A_\xi$ for $\xi < \alpha$ are.
	
At any successor stage $\alpha + 1$, if $x_\alpha \in \dom(A_\alpha)$, set $A_{\alpha + 1} = A_\alpha$. Otherwise, we must extend $A_\alpha$ to $x_\alpha$ while remaining 2-consistent with $\mathcal{C}$. Assume for contradiction that neither $B_0 := A_\alpha \cup \{x_\alpha \mapsto 0 \}$ nor $B_1 := A_\alpha \cup \{x_\alpha \mapsto 1 \}$ are 2-consistent with $\mathcal{C}$. Then there are $y_0$, $y_1 \in \dom(A_\alpha)$ such that $B_0 |_{\{y_0, x_\alpha\}}$ and $B_1 |_{\{y_1, x_\alpha \}}$ have no extension in $\mathcal{C}$. But $A_\alpha |_{\{y_0, y_1\}}$ has an extension $B$ in $\mathcal{C}$, and $B$ must be an extension of either $B_0 |_{\{y_0, x_\alpha\}}$ or $B_1 |_{\{y_1, x_\alpha \}}$, a contradiction. So $A_\alpha$ has a 2-consistent extension to $x_\alpha$, and we set $A_{\alpha + 1}$ to be such an extension.
	
We then take $\cup_{\xi < |X|} A_\xi$ to be our total extension.
\end{proof}

As the following examples show, consistency dimension and strong consistency dimension may differ when $\C(\mathcal{C}, \mathcal{H}) \geq 3$.

\begin{exam}
	Let $X = \{a,b,c,d,e\}$. Let
	\[
	\mathcal{C} = \mathcal{H} = \left\{ \{a,b,c\}, \{a,b,d\}, \{a,c,d,e\}, \{b,c,d,e\} \right\}.
	\]
	One can verify that $\C(\mathcal{C}, \mathcal{H}) = 3$, but the partially specified subset $\{a,b,c,d\}$ with $e$ unspecified witnesses that $\SC(\mathcal{C}, \mathcal{H}) > 3$.
\end{exam}

\begin{exam} \label{exam:exponential2}
	Continuing Example \ref{example:exponential}, observe that $\SC(\mathcal{C}, \mathcal{C}) = c^d$. In particular, the partially specified subset $A'$ given by
	\[
	A'(a_\tau) = \begin{cases}
	0 & |\tau| = d \\
	\text{undefined} & \text{otherwise}
	\end{cases}
	\]
	witnesses that $\SC(\mathcal{C}, \mathcal{C}) > c^d - 1$. Then we learn in at most $SC(\mathcal{C}, \mathcal{C})$ many queries. Moreover, this demonstrates that consistency dimension and strong consistency dimension can differ by an arbitrarily large amount (allowing $\Ldim(\mathcal{C})$ to vary), and that strong consistency dimension may even be exponentially larger than consistency dimension.
	
\end{exam}

Strong consistency dimension, like consistency dimension, categorizes equivalence query learning:

\begin{thm}\label{EQCharCon}
	$\mathcal{C}$ is learnable with equivalence queries from $\mathcal{H}$ iff $\Ldim(\mathcal{C}) \leq \infty$ and $\SC(\mathcal{C}, \mathcal{H}) < \infty$.  In particular, $\SC(\mathcal{C}, \mathcal{H}) \leq \LC^{EQ}(\mathcal{C}, \mathcal{H})$.
\end{thm}

\begin{proof}
	For the reverse direction, use Theorem \ref{EQlearn} and the observation that $\C(\mathcal{C}, \mathcal{H}) \leq \SC(\mathcal{C}, \mathcal{H})$.
	
	For the forward direction, use Propositions \ref{stableEQlearningeasy} and \ref{prop:Clowerbound}. In particular, if $\SC(\mathcal{C}) \geq c$, then there is a partially specified subset $A$ that is $(c-1)$-consistent with $\mathcal{C}$ but which has no total extension in $\mathcal{H}$. Then, by Proposition \ref{prop:Clowerbound}, $c \leq \LC^{EQ}(\mathcal{C}, \mathcal{H})$.
\end{proof}


\begin{cor} \label{CandSC}
	Suppose $\Ldim(\mathcal{C}) < \infty$. Then $\C(\mathcal{C}, \mathcal{H}) < \infty$ iff $\SC(\mathcal{C}, \mathcal{H}) < \infty$.
\end{cor}



The distinction between consistency dimension and strong consistency dimension is subtle, and many previous results hold with little to no modification if one replaces consistency dimension with strong consistency dimension.  On the other hand, our work in section \ref{applications} will reveal the practical difficulties associated with with strong consistency dimension in complicated concept classes. 

We have already seen in Theorem \ref{EQCharCon} that strong consistency dimension provides a better lower bound for learning complexity. It is also known in the finite case that strong consistency dimension also gives a stronger upper bound for learning complexity:
\begin{thm}\cite[Theorem 2]{balcazar2002consistencydimension} \label{thmFiniteEQlearningBalcazar}
	Suppose $\mathcal{C}$ is finite. Then $\LC^{EQ}(\mathcal{C}, \mathcal{H}) \leq \lceil SC(\mathcal{C}, \mathcal{H}) \cdot \ln |\mathcal{C}| \rceil$. 
\end{thm}

\begin{proof}
As this was originally framed in the setting where concepts were represented by strings, we give an abbreviated translation of the original proof into the language of set systems. This proof demonstrates the utility of constructing a partial hypothesis and taking some complete extension.

Let $c = \SC(\mathcal{C}, \mathcal{H})$. At stage $i$, let $\mathcal{C}_i \subseteq \mathcal{C}$ be the set of remaining possible target concepts. Let $A_i$ be the partially specified subset given by
\[
	A(x) = \begin{cases}
	1 & x \text{ belongs to more than } \frac{c - 1}{c}|\mathcal{C}_i| \text{ many } C \in \mathcal{C}_i \\
	0 &  x \text{ belongs to less than } \frac{1}{c}|\mathcal{C}_i| \text{ many } C \in \mathcal{C}_i \\
	\text{undefined} & \text{otherwise.}
	\end{cases}
\]
Observe that $A$ is $c$-consistent with $\mathcal{C}$---given any $Y := \{x_0, \ldots, x_{c-1}\} \subseteq \dom(A)$, for each $j$, less than $\frac{1}{c}|\mathcal{C}_i|$ many remaining concepts disagree with $A$ on $x_j$, so less than $c\frac{1}{c}|\mathcal{C}_i| = |\mathcal{C}_i|$ many concepts disagree with $A$ on some $x_j$. So some concept agrees with $A$ on $Y$. So $A$ is $c$-consistent.
	
So we can find some $B \in \mathcal{H}$ such that $B \sqsupseteq A$, and we submit $B$ as our hypothesis. By choice of $A$, if we receive a counterexample, we will have $|\mathcal{C}_{i+1}| \leq \frac{c-1}{c}|\mathcal{C}_i|$. Repeating this $\lceil c \cdot \ln |\mathcal{C}| \rceil$ many times is enough to identify and submit the target concept.
\end{proof}

In light of Example \ref{exam:exponential2}, one hopes that improved bounds on learning can be found in terms of strong consistency dimension and Littlestone dimension when $\mathcal{C}$ is infinite. We are unable to show this presently, but offer some evidence in this direction:

\begin{prop}\label{EQlearnSC2}
	Suppose $\Ldim(\mathcal{C}) = d < \infty$ and $\SC(\mathcal{C}, \mathcal{H}) = 2 < \infty$. Then $\LC^{EQ}(\mathcal{C}, \mathcal{H}) = d + 1$.
\end{prop}
\begin{proof} We know by Proposition \ref{stableEQlearningeasy} that $d+1$ is a lower bound. We show that it is also an upper bound.
	
Let $V_0 = \mathcal{C}$. Inductively define $V_i$, $i = 1, \ldots, d$ as follows. Given $V_i$, for any $x \in X$ and $j \in \{0,1\}$, let 
\[
	V_i^{(x, j)} := \{ B \in V_i \, | \, \chi_B(x) = j \},
\]
where $\chi_B$ is the characteristic function on $B$. Construct the partially specified subset $A_i$ where
\begin{equation}
	A_i(x) = \begin{cases}
	0 & \Ldim(V_i^{(x,0)}) = \Ldim(V_i) \\
	1 & \Ldim(V_i^{(x,1)}) = \Ldim(V_i) \\
	\text{undefined} & \text{otherwise.}
	\end{cases} \label{eqn:EQ2}
\end{equation}
We claim that $A_i$ has an extension in $H$. By our assumption that $\SC(\mathcal{C}, \mathcal{H}) = 2$, it suffices to check that $A$ is 2-consistent with $V_i$. Suppose for contradiction that there are $a_0, a_1 \in \dom(A_i)$ such that, without loss of generality, $A_i(a_0) = A_i(a_1) = 0$, but there is no extension of $A_i|_{\{a_0, a_1\}}$ in $V_i$. Then observe that $V_i^{(x_0,0)} \subseteq V_i^{(x_1, 1)}$, whence 
\[
	\Ldim(V_i) \geq \Ldim(V_i^{(x_1, 1)}) \geq \Ldim(V_i^{(x_0, 0)}) = \Ldim(V_i),
\]
so $\Ldim(V_i^{(x_1, 1)}) = \Ldim(V_i)$. But we also have $\Ldim(V_i^{(x_1, 0)}) = \Ldim(V_i)$, a contradiction, as we could then construct a binary element tree with proper labels from $V_i$ of height $\Ldim(V_i) + 1$ with $x_1$ at the root.
	
Let $B_i \in \mathcal{H}$ be a total extension of $A_i$. Submit $B_i$ as the hypothesis. If $B_i$ is correct, we are done. Otherwise, we receive a counterexample $x_i$. Set
\[
	V_{i+1} := \{ B \in V_i \, | \, \chi_B(x_i) \neq \chi_{B_i}(x_i) \}.
\]
Observe that at each stage, $\Ldim(V_{i+1}) < \Ldim(V_i)$. Therefore, if we make $d$ queries without correctly identifying the target, then we must have $\Ldim(V_d) = 0$. Then $V_d$ is a singleton, which must be the target concept.
 
\end{proof}

The proof of Proposition \ref{EQlearnSC2} uses strong consistency in a key way, as the hypothesis is generated by extending a certain partially specified subset. Nevertheless, the conclusion holds under the assumption that $C(\mathcal{C}, \mathcal{H}) = 2$, due to Proposition \ref{prop:C-SC12equal}.

\subsection{Adding membership queries and efficient learning of finite classes} \label{EQMQMQ}

Consistency dimension was originally derived from the notion of polynomial certificates, which was used to characterize learning with equivalence and membership queries in the finite case by \cite{hellerstein1996queries}. 
The following is an improvement of the upper bound on EQ+MQ learning complexity of $\lceil \C(\mathcal{C}, \mathcal{H}) \log_2 |\mathcal{C}| \rceil$ implicit in the proof of Theorem 3.1.1 in \cite{hellerstein1996queries} (stated explicitly in \cite{balcazar2002consistencydimension}). Our bound replaces $\log_2|\mathcal{C}|$ with $\Ldim(\mathcal{C})$.

\begin{thm} \label{EQMQalg} 
	Suppose $\Ldim(\mathcal{C}) = d < \infty$ and $\C(\mathcal{C}, \mathcal{H}) = c < \infty$. Then $\LC^{EQ + MQ}(\mathcal{C}, \mathcal{H}) \leq c'd + 1$, where $c' = \max \{1, c-1\}$.
\end{thm}

\begin{proof} 
\footnote{The algorithm is similar to that of Theorem \ref{EQlearn}. However, the applications of Lemma \ref{unionlemma} are replaced with membership queries.}
We proceed by induction on $d$. The base case, $d = 0$, is trivial, as then $\mathcal{C}$ is a singleton.
Suppose there is some element $x$ such that $\Ldim(\mathcal{C} \cap x) < d + 1$ and $\Ldim(\mathcal{C} \setminus x) < d + 1$, where $\mathcal{C} \cap x := \{ A \in \mathcal{C} \, | \, x \in A \}$ and $\mathcal{C} \setminus x := \{ A \in \mathcal{C} \, | \, x \notin A \}$. Then by induction, any concept in $\mathcal{C} \cap x$ can be learned in at most $c'd + 1$ queries with guesses from $\mathcal{H}$, and the same is true for $\mathcal{C} \setminus x$. Submit $x$ as a membership query. This tells us whether the target concept lies in $\mathcal{C} \cap x$ or $\mathcal{C} \setminus x$, and then we require at most $c'd + 1$ many queries, for a total of $c'd + 2 \leq c'(d+1) + 1$ many queries.
	
If no such $x$ exists, then for all $x$, either $\Ldim(\mathcal{C} \cap x) = d + 1$ or $\Ldim(\mathcal{C} \setminus x) = d + 1$. Let $B$ be such that $x \in B$ iff $\Ldim(\mathcal{C} \cap x) = d+1$.
	
If $B \in \mathcal{H}$, then we submit $B$ as our query. If we are incorrect, then by choice of $B$, the class $\mathcal{C}'$ of concepts consistent with the counterexample $x_0$ will have Littlestone dimension $\leq d$. By induction, any concept in $\mathcal{C}'$ can be learned in at most $c'd + 1$ many queries, and so we learn the target in at most $c'd + 2 \leq c'(d+1) + 1$ queries.
	
If $B \notin H$, then, since $\C(\mathcal{C}, \mathcal{H}) = c$, there are some $x_0, \ldots, x_{c-1}$ such that there is no $A \in \mathcal{C}$ such that $B|_{\{x_0, \ldots, x_{c-1}\}} \sqsubseteq A$. (Observe that this cannot happen when $c = 1$. In fact, Proposition \ref{prop:C-SC12equal} and the proof of Proposition \ref{EQlearnSC2} imply that this cannot even happen when $c = 2$. In particular, $c' = c-1$.) Then, with notation as in the proof of Proposition \ref{stableEQlearningeasy},
\[
	\mathcal{C} = (\mathcal{C}^{(x_0, 1-B(x_0))}) \cup \ldots \cup (\mathcal{C}^{(x_{c-1}, 1-B(x_{c-1}))}),
\]
and $\Ldim(\mathcal{C}^{(x_i, 1-B(x_i))}) \leq d$ for each $i$. By induction, any concept in each $\mathcal{C}^{(x_i, 1-B(x_i))}$ can be learned in at most $c'd + 1$ many queries. By submitting $x_0, \ldots, x_{c-2}$ as membership queries, we can determine some $i$ such that the target belongs to $\mathcal{C}^{(x_i, 1-B(x_i))}$ (if the result of each membership query on $x_j$ is $B(x_j)$, then we know that $i = c-1$). We therefore learn in at most $c'd + 1 + (c-1) = c'(d+1) + 1$ many queries.
\end{proof} 

We have a lower bound on learning complexity in terms of consistency dimension in this setting analogous to Proposition \ref{prop:Clowerbound}: 

\begin{prop} \label{goodlower} 
	Suppose there is some (total) subset $A$ which is $n$-consistent with $\mathcal{C}$ but which does not have a total extension in $\mathcal{H}$. Then $n < \LC^{EQ+ MQ}(\mathcal{C}, \mathcal{H})$. In particular, $\C(\mathcal{C}, \mathcal{H}) \leq \LC^{EQ + MQ}(\mathcal{C}, \mathcal{H})$.
\end{prop}

\begin{proof} 
We first show that $n < \LC^{EQ+ MQ}(\mathcal{C}, \mathcal{H})$. If the learner submits $x$ as a membership query, the teacher returns $A(x)$ if possible, that is, if there is a concept $B \in \mathcal{C}$ which agrees with the previous data and satisfies $B(x) = A(x)$.
	
By hypothesis, given any equivalence query $H$, the teacher can find some $x \in \dom(A)$ such that $H(x) \neq A(x)$, and the teacher returns a counterexample of this form if possible, that is, if there is a concept $B \in \mathcal{C}$ which agrees with the previous data and satisfies $B(x) = A(x)$.
	
Moreover, since $A$ is $n$-consistent with $\mathcal{C}$, the teacher is able to return data of this form for the first $n$ queries. Thus $\mathcal{C}$ cannot be learned with fewer than $n + 1$ equivalence queries from $\mathcal{H}$. 

From this, it follows that $\C(\mathcal{C}, \mathcal{H}) \leq \LC^{EQ + MQ}(\mathcal{C}, \mathcal{H})$. 
\end{proof} 

Finally, putting together the various upper and lower bounds from this section we give a characterization of those problems efficiently learnable by equivalence and membership queries: 

\begin{thm} \label{polylowereqmq} Let $(\mc C_n, \mc H_n)_{n \in \mathbb{N}}$ be a family of concept classes and hypothesis classes, respectively. Let $c_n = \C( \mc C_n , \mc H_n ).$ Let $d_n = \Ldim (\mc C_n ).$ The following are equivalent:
\begin{description}
	\item[(i)] $\LC^{EQ+MQ}(\mathcal{C}_n, \mathcal{H}_n)$ is bounded by a polynomial in $n$.
	\item[(ii)] $c_n$ and $d_n$ are bounded by a polynomial in $n$.
	\item[(iii)] The algorithm from Theorem \ref{EQMQalg} learns $\mc C_n$ in at most polynomially in $n$ many membership queries and equivalence queries in $\mc H_n$.
\end{description}
\end{thm} 

\begin{proof}
(ii) $\Rightarrow$ (iii) follows immediately from Theorem \ref{EQMQalg}, and (iii) $\Rightarrow$ (i) follows by definition of learning complexity.


(i) $\Rightarrow$ (ii): In Proposition \ref{goodlower}, we showed that $\LC^{EQ+MQ} ( \mc C  , \mc H ) \geq \C(\mc C, \mc H),$ so it follows that if $c_n$ is not polynomially bounded then neither is $\LC^{EQ+MQ} ( \mc C _n , \mc H _n)$. 

Now suppose that $d_n$ is not polynomially bounded. By \cite[Theorem 2.1]{auer1994simulating} \footnote{The inequality of \cite{auer1994simulating} gives a lower bound for $\LC^{EQ+MQ}$ which improved on the lower bound of $\frac{\LC ^{EQ} (\mc C , \mc P(X))}{\log (1+\LC ^{EQ} (\mc C , \mc P(X)))}$ from  \cite[Theorem 3]{maass1990complexity}. In fact, Theorem 3 of \cite{maass1990complexity} actually suffices for our purposes.} we have

$$\LC^{EQ+MQ}( \mc C , \mc H ) \geq \LC^{EQ+MQ}( \mc C , \mc P(X)) \geq \log \left(\frac{4}{3} \right) \cdot \LC^{EQ} (\mc C , \mc P(X)).$$

By \cite[Theorems 5 and 6]{littlestone1988learning}, we can replace $\LC ^{EQ} (\mc C , \mc P(X))$ with $\Ldim(\mc C).$ Thus: 
\[
	\LC^{EQ+MQ}( \mc C_n , \mc H_n ) \geq \log \left(\frac{4}{3} \right) \cdot d_n,
\] 
from which it follows that $\LC^{EQ+MQ}( \mc C_n , \mc H_n )$ is not polynomially bounded. 
\end{proof}


Finally, the upper and lower bounds of this section also yield a characterization of which infinite classes are learnable in finitely many equivalence and membership queries. 
\begin{cor}
	$\LC^{EQ+MQ}(\mathcal{C}, \mathcal{H}) < \infty$ iff $\Ldim(\mathcal{C}) < \infty$ and $\C(\mathcal{C}, \mathcal{H}) < \infty$.
\end{cor}

\subsection{The negation of the finite cover property} \label{NFCP}

One can compare set systems with finite strong consistency dimension, to the model-theoretic classes of formulas and theories without the \emph{finite cover property}, which we define below. Informally, the negation of the finite cover property allows for a specific quantitative bound for applications of compactness. 

\begin{defn} 
	Fix a first order theory $T$. A formula $\phi(x;y)$ in the language of $T$ \emph{does not have the finite cover property (ncfp)} if there is $n = n(\phi)$ such that for all $\mathcal{M} \models T$, and every $p \subseteq \{\phi(x;a), \neg \phi(x;a) \, | \, a \in M\}$, the following holds: if every $q \subseteq p$ of size $n$ is consistent, then $p$ is consistent. We let $n(\phi)$ denote the minimal such $n$.
	
	$T$ does not have the finite cover property if all formulas $\phi(x;y)$ do not have the finite cover property.\footnote{The definition of nfcp on the formula level given here is stronger than original formulation in \cite{shelah1990classification}, but it gives an equivalent characterization on the level of theories.} 
\end{defn} 

Consider the setting where $\mathcal{C}$ is generated by a formula $\phi(x; y)$, that is, $\mathcal{M} \models T$ and $$\mathcal{C} = \mathcal{C}_\phi = \{ \phi(M;b) \, | \, b \in M \} .$$ That is, $\mathcal{C}_\phi$ consists of the $\phi$-definable sets. Suppose $\phi^{opp}(y;x) = \phi(x;y)$ does not have the finite cover property, witnessed by some $n = n(\phi^{opp})$. Then, given any disjoint $A_0, A_1 \subseteq X$, if every size $n$ subset of
\[
p(y) := \{\phi^{opp}(y;a) \, | \, a \in A_1 \} \cup \{\neg \phi^{opp}(y;a)\, | \, a \in A_0 \}
\]
is consistent, then $p(y)$ is consistent. We can identify this partial type with the partially specified subset $A$ where
\[
	A(x) = \begin{cases}
		0 & x \in A_0 \\
		1 & x \in A_1 \\
		\text{unspecified} & \text{otherwise.}
	\end{cases}
\]

By passing to an $|M|^+$-saturated extension $\mathcal{N} \succ \mathcal{M}$ to obtain a larger parameter set, we can find $b' \in \mathcal{N}$ satisfying $p(y)$. Then $\phi(M, b')$ is a total extension of $A$. 

Supposing we have passed to an $|\mathcal{M}|^+$ saturated extension $\mathcal{N}$, we can let $$\mathcal{H} = \mathcal{H}_\phi := \{ \phi(M;b') \, | \, b' \in N\} .$$ That is, $\mathcal{H}_\phi$ consists of all externally $\phi$-definable subsets of $M$, as $N$ contains realizations of all consistent partial $\phi^{opp}$-types over $M$. By the compactness theorem, this means that $N$ contains realizations of all finitely consistent partial $\phi^{opp}$-types over $M$. Having identified partially specified subsets of $M$ with their corresponding $\phi^{opp}$-type, this amounts to observing that $\mathcal{H}_\phi$ contains total extensions of all finitely consistent partially specified subsets, equivalently, contains all finitely consistent total subsets. 

This gives a model-theoretic motivation to the strategy suggested by Proposition \ref{prop:SCproperties}. Adding all finitely consistent subsets to $\mathcal{H}$ amounts to saturating $\mathcal{N}$ so as to realize all $\phi^{opp}$-types over $\mathcal{M}$.

If $\phi^{opp}$ has nfcp with $n(\phi)=n$, then the finitely consistent partial types are exactly the $n$-consistent types. Then $\mathcal{H}_\phi$ contains total extensions of all $n$-consistent partially specified subsets, so $SC(\mathcal{C}_\phi, \mathcal{H}_\phi) = n$. Note that $\phi^{opp}$-types witnessing that $n$ is the minimal such $n$ at which $\phi^{opp}$ has nfcp give partially specified subsets witnessing that $\SC(\mathcal{C}_\phi, \mathcal{H}_\phi) \not < n$. This reflects a variant of Proposition \ref{prop:SCproperties} for strong consistency dimension.

In particular, formulas $\phi$ such that $\phi^{opp}$ has nfcp provide a rich family of examples where $\mathcal{C}_\phi$ has finite (strong) consistency threshold. That is, for such $\phi$, it is necessary and sufficient for $\mathcal{H}$ to contain all externally $\phi$-definable subsets (that is, all total finitely consistent partially specified subsets) to obtain $\SC(\mathcal{C}, \mathcal{H}) < \infty$. On the other hand, when $\phi^{opp}$ has the finite cover property, the externally definable sets are no longer sufficient, and one must venture beyond the sets $\phi$ is capable of cutting out to obtain $\SC(\mathcal{C}, \mathcal{H}) < \infty$ (that is, by adding some sets which are inconsistent). 

Furthermore, Littlestone dimension of $\phi(x;y)$ (that is, the Littlestone dimension of $\mathcal{C}_\phi$) is expressible as a first-order property. So we will have $\Ldim(\mathcal{C}) = \Ldim(\mathcal{H})$. So when the context is a set system $\mathcal{C}$ generated by a stable formula $\phi(x;y)$ with $\phi^{opp}(y;x)$ nfcp, we can obtain a set system $\mathcal{H}$ such that $\SC(\mathcal{C}, \mathcal{H}) < \infty$, but $\mathcal{H}$ is not much more complicated than the original set system - $\mc H$ has the same Littlestone dimension as $\mc C$. This is essentially the content of Theorem \ref{thm:fixedLdim} when $\mathcal{C}$ has finite consistency threshold.

We give an example from model theory where $\phi^{opp}$ has the fcp.

\begin{exam} 
	Let $\mathcal{M}$ be a structure in the language $\{E\}$, where $E$ is an equivalence relation with one class of size $n$ for each $n \in \mathbb{N}$, possibly with some infinite classes. Let
	\[
		\phi(x;y) \quad \text{be} \quad E(x,y) \wedge x \neq y
	\]
	and let $\mathcal{C} = \mathcal{C}_\phi$.
	
	Suppose $a_1, \ldots, a_d$ are the elements belonging to the equivalence class of size $d$. Then the $\phi^{opp}$-type $\{\phi(a_i, y) \, | \, i \leq d \}$ $(d-1)$-consistent but $d$ inconsistent. Since there are equivalence classes of arbitrarily large size, these witness that $\phi^{opp}$ is not nfcp. One can check that $\Ldim(\mathcal{C}_\phi) = 2$.
	
	In any $|M|^+$-saturated elementary extension $\mathcal{N}$ of $\mc M$, no additional elements are added to the finite equivalence classes already present in $\mc M$, though $\mathcal{N}$ adds new infinite classes and new elements to any existing infinite classes. 
	
	An attempt to learn $\mathcal{C}_\phi$ by equivalence queries following the strategy of Theorem \ref{EQlearn} would be as follows. We are attempting to identify some $c \in M$. Letting $a_0$ be an element in a new infinite equivalence class in $\mathcal{N}$, we guess $\phi(M, a_0) = \emptyset$. Then any counterexample will identify an element belonging to the equivalence class of $c$. If $c$ belongs to an infinite class, then we can find some $a_1 \in N$ which is a new element of this class. Then $\phi(M, a_1) = \{b \in M \, | \, E(b, c)\}$. Then $c$ is the only available counterexample, and we submit the correct concept $\phi(M, c)$ at our next turn. However, if $c$ belongs to the finite class of size $n$, then $N$ has no new elements in this class. Then the relevant queries, which are of the form $\phi(M, a)$ for $a$ in the class of $c$, are already present in $\mathcal{C}_\phi$. Then we are essentially attempting to identify a singleton from a set of size $n$, and it is clear that the process could take up to $n$ additional guesses. 
\end{exam}


\section{Efficient learnability of regular languages} \label{applications} 
In a seminal paper, \cite{angluin1987learning} showed that regular languages are efficiently learnable with equivalence queries plus membership queries, and in this subsection, we will use Theorem \ref{EQMQalg} to give an alternate short proof of this fact.\footnote{In the following sections, we only make use of \emph{proper equivalence queries}, that is, $\mathcal{H} = \mathcal{C}$. We shall therefore let $\C(\mathcal{C}) := \C(\mathcal{C}, \mathcal{C})$, which we will call the consistency dimension of $\mathcal{C}$ (with analogous notation for strong consistency dimension).} Let $\mc L_{n,m}$ be the class of binary regular languages on strings of length at most $m$ specified by a deterministic finite automaton on at most $n$ nodes. The $\mc L^*$ algorithm of \cite{angluin1987learning} specifically uses $\mc O(n)$ equivalence queries and $\mc O(mn^2)$ membership queries. We let $DFA_2(n)$ denote the collection of (equivalence classes of) deterministic finite automata accepting binary strings and having at most $n$ nodes. The proof of the next proposition is straightforward. 

\begin{prop} \label{LDDFA} The Littlestone dimension of $DFA_2(n)$ is at most $ o (1) (n \log n)$.
\end{prop} 
\begin{proof} 
In \cite[Proposition 1]{ishigami1997vc}, it is shown that $|DFA_2 (n)| \leq \frac{n^{2n} 2^n  n }{n!} \leq 2^ { o (1)( n \log n)  }.$ From this, it follows that the Littlestone dimension of $DFA_2 (n)$ is at most $ o(1) ( n \log n)$. \end{proof} 

The proof of the following proposition reveals the connection between consistency and the Myhill-Nerode theorem. 

\begin{prop} 
	\label{CDDFA} $\C(DFA_2(n)) \leq 2\binom{n+1}{2} = n(n+1)$. 
\end{prop} 

\begin{proof} 
Fix a subset $C$ of binary strings and $x,y$ binary strings. We say that $z$ is a ($C$-) distinguishing extension of $x$ and $y$ if $x z \in C $ but $y z \notin C$ or vice versa. If $x$ and $y$ have no distinguishing extension, then we say $x$ and $y$ are $C$-equivalent, and write $x \sim_C y$. The Myhill-Nerode theorem \cite{nerode1958linear} says that a subset of binary strings of length $m$ is the accept set of a finite automaton with at most $n$ nodes if and only if the number of $\sim_C $ classes is at most $n$. Thus, given any subset $C$ of the binary strings of length $m$ which is not a regular language recognized by an automaton with at most $n$ nodes, there are at least $n+1$ $\sim _C $-classes of elements. Pick representatives $x_0, \ldots, x_n$ from $n+1$ classes, and for each $i < j$, pick some $z_{ij}$ that is a distinguishing extension of $x_i$ and $x_j$. Then restricting $C$ to the partial assignment on $\{ x_k z_{ij} \, | \, i < j, \, k = i, j \}$, a domain of size $2 \binom{n+1}{2} = n(n+1)$ that witnesses that $x_i \not \sim _C x_j$ for all $i \neq j$, we can see that this restriction is inconsistent with the class of regular languages recognized by automata with at most $n$ nodes. Therefore $\C(DFA_2(n)) \leq n(n+1)$. \footnote{Note that the same proof shows that the consistency dimension of $DFA_m (n)$ is also at most $n(n+1)$.}
\end{proof}

Now, by Theorem \ref{EQMQalg} and the previous two results, it follows that:

\begin{thm} The class $\mc L_{n,m}$ is learnable in at most $o(1) n \log n$ equivalence queries and at most $o(1)\left( n \log n \right) (n(n+1))$ membership queries. 
\end{thm} 

It is interesting to note that contrary to $\mc L^*$, when using the algorithm from Theorem \ref{EQMQalg}, there is no dependence on $m$, the length of the binary strings which the teacher is allowed to provide as counterexamples\footnote{We should also note that $\mc L^*$ was improved by Schapire to give a better bound on membership queries (still depending on $m$). \cite{schapire1991design}.}.

Theorem \ref{EQlearn} now implies that $\mc L_{n,m}$ is learnable in at most $( n(n+1) )^ {(o(1) n \log n)}$ equivalence queries. Theorem \ref{thmFiniteEQlearningBalcazar} shows that a finite class $\mc C$ is learnable in at most $\lceil \SC(\mc C) \cdot \ln|\mc C| \rceil$ equivalence queries. Since \cite{angluin1990negative} showed that $\mc L_{n,m}$ is not learnable in polynomially many equivalence queries, it follows that $\SC(\mc L_{n,m})$ cannot be polynomial in $n,m$. 

\subsection{Learning $\omega$-languages} \label{omegaLanguage}
In this section, we consider the natural extension to languages on infinite strings indexed by $\omega$, called $\omega$-languages. For an alphabet $\Sigma$, we denote by $\Sigma ^ \omega$, the strings of symbols from $\Sigma$ of order type $\omega$. Similar to the previous section, we consider an automaton, which consists of the collection $\m A = (\Sigma, Q, q_0, \delta),$ where $Q$ is a finite collection of states, $q_0$ is the initial state, and $\delta: Q \times \Sigma \rightarrow 2^Q$ is a transition rule. To form a language, an automaton is equipped with an acceptance criterion.\footnote{Numerous acceptance criteria have been extensively studied in the literature, and we refer the reader to \cite{angluin2016learning,fisman2018families,fisman2018inferring} for overviews.} Fix a subset $F \subseteq Q$. A run of a \emph{B\"uchi automaton} is accepting if and only if it visits the set $F$ infinitely often. An $\omega$-language is $\omega$-regular if it is recognized by a non-deterministic B\"uchi automaton. A run of a \emph{co-B\"uchi automaton} is accepting if and only if it visits $F$ only finitely often. Let $\psi: Q \rightarrow \{ 1, \ldots , k \}$ be a function, which we think of as a coloring of the states of the automaton. Let $c$ be the minimum color which is visited infinitely often. A run of a \emph{parity automaton} is accepting if and only if $c$ is odd.

Two $\omega$-regular languages are equivalent if they agree on the set of periodic words \cite{mcnaughton1966testing}, which allows for the possibility of recognizing the $\omega$-language using finitary automata. This is the approach of \cite{angluin2016learning, fisman2018families}, whose notation we follow closely. A \emph{family of DFAs} (FDFA) $\mc F$ is a pair $(Q,P)$ where $Q$ is a DFA with $|Q|$ states and $P$ is a collection of $|Q|$ many DFAs, which we refer to as \emph{progress DFAs} - one DFA $P_q$ for each state $q$ of $Q$. Given a pair of finite words, $(u,v)$, a run of our family of DFAs consists of running $Q$ on $u$, then running $P_{Q(u)}$ on $v$ where $Q(u)$ is the ending state of $Q$ on $u$. The pair $(u,v)$ can be used to represent an infinite periodic word $uv^\omega$. 

Let $FDFA(n,m)$ be the class of families of deterministic finite automata where the leading automaton has at most $n$ nodes and the progress automata each have at most $m$ nodes. It is \emph{not} quite true that once an $\omega$-regular language has been reduced to an FDFA that one can use $\mc L^*$ directly to learn the various DFAs in the family \cite[section 4]{angluin2016learning}. It is also not completely obvious what the bounds for Littlestone and consistency dimension are in terms of the DFAs in the family, but the next two results give such bounds which imply the efficient learnability of $\omega$-regular languages. 

\begin{prop} \label{LDFDFA}
	The class $FDFA(n,m)$ has Littlestone dimension at most $ o(1) (n \log n + n m \log m ).$ 
\end{prop} 
\begin{proof} 
	The number of FDFAs of size $(n,m)$ is clearly at most $|DFA_2(n)| \cdot |DFA_2(m)|^n.$ That is 
	$$|FDFA (n,m) |  \leq |DFA_2 (n) | \cdot |DFA_2(m)|^n.$$ It follows that $$\Ldim(FDFA(n,m)) \leq \log (|DFA_2 (n) | \cdot |DFA_2(m)|^n) $$ and using \cite[Proposition 1]{ishigami1997vc}, the desired bound follows.  
\end{proof}

\begin{prop} \label{CDFDFA}
	$\C(FDFA(n,m)) \leq 2\binom{n(m+1)}{2} = \mc O(n^2m^2)$.
\end{prop} 
\begin{proof} A run of an FDFA on $(u,v)$ can be simulated by the run of an appropriate automaton in the class $DFA_3 (n \cdot (m+1)).$ To see this, input word $u \$ v$ where $\$$ is a new symbol (recall we are assuming $u,v$ are binary) to a DFA which has the same diagram as the FDFA but with an edge labeled with $\$$ from each state of the leading automaton to the initial state of the corresponding progress DFA. Now it follows by Proposition \ref{CDDFA} that the consistency dimension of $FDFA(n,m)$ is at most $2\binom{n(m+1)}{2}.$ 
\end{proof}

Using the previous two results together with Theorem \ref{EQMQalg}, one can deduce the efficient learnability of $FDFA(n,m)$: 

\begin{thm} The class $FDFA(n,m)$ is learnable in at most $o (1) (n \log n +n \cdot m \log m )$ equivalence queries and at most $o (1) (\log n + m \log m ) \cdot n^3m^2$ membership queries. 
\end{thm} 

We have formulated our bounds in terms of the number of states in the FDFA corresponding to a given $\omega$-language. In \cite{angluin2016learning,fisman2018families} bounds on the number of states of FDFAs in terms of the number of states of automata for $\omega$-languages with various acceptors are given. Specifically, the following bounds hold: 

\begin{enumerate} 
\item When $\mc A$ is a deterministic B\"uchi (DBA) or co-B\"uchi (DCA) automaton with $n$ states, there is an equivalent FDFA of size at most $(n,2n)$ \cite[5.3]{fisman2018families}. 
\item When $\mc A$ is a deterministic partiy automaton (DPA) with $n$ states and $k$ colors, there is an equivalent FDFA of size at most 
$(n, kn)$ \cite[5.4]{fisman2018families}. 
\item When $\mc A$ is an nondeterministic B\"uchi automaton (NBA) with $n$ states, there is an equivalent FDFA of size at most $(2 ^ {\mc O (n \log n ) }, 2 ^ {\mc O (n \log n)})$. 
\end{enumerate} 

Any NBA can be translated into a DPA, and so 2) yields the efficient learnability of $\omega $-regular languages \emph{in terms of the number of states in a DPA} (this translation also yields 3). However, the translation from NBA to DPA is known to require an exponential increase in the number of states in general \cite{piterman2006nondeterministic}. From an FDFA of size at most $(n,k)$ there is a translation into an NBA with at most $\mc O ( n^2 k^3)$ states \cite[Theorem 5.8]{fisman2018families}, and so it follows that the exponential increase in states in moving from NBAs to FDFAs is necessary \cite[Theorem 5.6]{fisman2018families}. 

Finally, we mention that \cite{angluin2018regular} define  restricted classes of $\omega$-languages for which right-congruence is \emph{fully informative}, and isolate numerous classes (e.g. for each type of acceptor from the previous subsection) of $\omega$-languages for which an infinitary invariant of the Myhill-Nerode theorem holds. This variant of Myhill-Nerode is sufficient to bound the consistency dimension (and thus establish the learnability) of the classes in terms of the number of of right equivalence classes of $\sim _{\mc L}$ similar to the proof of Proposition \ref{CDDFA}.


\section{Random counterexamples and EQ-learning} \label{randomeq}

In section \ref{basicEQ} we characterized learnability by equivalence queries in terms of Littlestone dimension and strong consistency dimension. The setting of equivalence query learning \cite{angluin1988queries} as described in section \ref{basicEQ} deals with worst-case bounds for algorithmic identification of concepts by a learner. In this section, we follow \cite{angluin2017power} and analyze a slightly different situation, in which the teacher selects the counterexamples at random, and we seek to bound the \emph{expected} number of queries. \cite{angluin2017power} worked specifically with concept classes coming from boolean matrices, which was convenient for their notation. Our formulation is equivalent, but we use slightly different notation. 

Throughout this section, let $X$ be a finite set, let $\mc C$ be a set system on $X$, and let $\mu$ be a probability measure on $X$. For $A,B \in \mc C$, let $\Delta (A,B) = \{x \in X \, | \, A(x) \neq B(x) \}$ denote the symmetric difference of $A$ and $B$. 

\begin{defn} We denote, by $\mc C_{\bar x = \bar i}$ for $\bar x \in X^n$ and $\bar i \in \{0,1\}^n$, the set system  
$ \{A \in \mc C \, | \, A(x_j )= i_j , \, j=1, \ldots , n \}.$ 
For $A \in \mc C$ and $a \in X$, we let $$u(A,a) = \Ldim (\mc C) - \Ldim (\mc C _{a = A(a)} ).$$ 
\end{defn}

For any $a \in X,$ either $\mc C_{a=1}$ or $\mc C_{a=0}$ has Littlestone dimension strictly less than that of $\mc C$ and so: 
\begin{lem} \label{ulemma} For $A, B \in \mc C$ and $a \in X$ with $A(a) \neq B(a),$ 

$$u(A,a ) + u (B, a) \geq 1.$$ 
\end{lem} 

Next, we define a directed graph which is similar to the \emph{elimination graph} of \cite{angluin2017power}. 

\begin{defn} 
We define the \emph{thicket query graph} $G_{TQ}(\mc C , \mu )$ to be the weighted directed graph on vertex set $\mc C$ such that the directed edge from $A$ to $B$ has weight $d(A,B)$ equal to the expected value of $\Ldim(\mc C)- \Ldim(\mc C _{x = B(x)}) $ over $x \in \Delta (A,B)$ with respect to the distribution $\mu |_{\Delta(A,B)}.$ \footnote{Here one should think of the query by the learner as being $A$, and the actual hypothesis being $B$. The teacher samples from $\Delta(A,B)$, and the learner now knows the value of the hypothesis on $x$.} 
\end{defn} 

\begin{defn} The \emph{query rank} of $A \in \mc C$ is defined as: 
$\inf _{B \in \mc C } (d(A,B)).$ 
\end{defn}

\begin{lem} \label{Lem13} For any $A \neq B \in \mc C$, $d(A, B)+d(B,A) \geq 1.$
\end{lem} 
\begin{proof} 
Noting that $\Delta (A,B) = \Delta (B,A),$ and using Lemma \ref{ulemma}:
\begin{eqnarray*} d(A,B) + d(B,A) & = & \sum _{a \in \Delta (A,B)} \frac{\mu(a)}{\mu (\Delta (A,B))} (u(A,a)+u(B,a)) \\
& \geq  & \sum _{a \in \Delta (A,B)} \frac{\mu(a)}{\mu (\Delta (A,B))} = 1. \end{eqnarray*} 
\end{proof}


\begin{defn} \cite[Definition 14]{angluin2017power} Let $G$ be a weighted directed graph and $ l \in \m N, \,  l >1.$ A \emph{deficient $l$-cycle} in $G$ is a sequence $v_0, \ldots v_{l-1}$ of distinct vertices such that for all $i \in [l]$, $d(v_i , v_{(i+1) \, (\mod l) } ) \leq \frac{1}{2}$ with strict inequality for at least one $i \in [l]$. 
\end{defn} 


The next result is similar to Theorems 16 (the case $l=3$) and Theorem 17 (the case $l >3$) of \cite{angluin2017power}, but our proof is rather different (note that the case $l=2$ follows easily from Lemma \ref{Lem13}).

\begin{thm} \label{nocycles} The thicket query graph $G_{TQ} (\mc C , \mu)$ has no degenerate $l$-cycles for $l \geq 2.$ 
\end{thm} 

The analogue of Theorem 16 can be adapted in a very similar manner to the technique employed by \cite{angluin2017power}. However, the analogue of the proof of Theorem 17 falls apart in our context; the reason is that Lemma \ref{ulemma} is analogous to Lemma 6 of \cite{angluin2017power} (and Lemma \ref{Lem13} is analogous to Lemma 13 of \cite{angluin2017power}), but our lemmas involve inequalities instead of equations. The inductive technique of \cite[Theorem 17]{angluin2017power} is to shorten degenerate cycles by considering the weights of a particular edge in the elimination graph along with the weight of the edge in the opposite direction. Since one of those weights being large forces the other to be small (by the \emph{equalities} of their lemmas), the induction naturally separates into two useful cases. In our thicket query graph, things are much less tightly constrained - one weight of an edge being large does not force the weight of the edge in the opposite direction to be small. However, the technique employed in our proof seems to be flexible enough to adapt to prove Theorems 16 and 17 of \cite{angluin2017power}.
\begin{proof} Suppose the vertices in the degenerate $l$-cycle are $A_0, \ldots , A_{l-1} $. 

By the definition of degenerate cycles and $d(-,-),$ we have, for each $i \in \m Z / l \m Z$, that $$\sum _{a \in \Delta (A_i, A_{i+1}) } \frac{\mu(a)}{\mu (\Delta (A_i ,A_{i+1}))} u(A_i,a) \leq \frac{1}{2},$$ so clearing the denominator we have 

\begin{equation} \label{firsteqn} \sum _{a \in \Delta (A_i,A_{i+1})} \mu(a) u(A_i,a) \leq \frac{1}{2} \mu (\Delta (A_i,A_{i+1})). \end{equation} 
\emph{Note that throughout this argument, the coefficients are being calculated modulo $l$.} Notice that for at least one value of $i$, the inequality in \ref{firsteqn} must be strict. 

Let $G, H$ be a partition of $$\mc X =  \{ A_1, \ldots , A_l \}.$$ Now define $$D (G, H) := \left \{a \in X \, | \, \forall A_1,B_1 \in G, \, \forall A_2, B_2 \in H,\, A_1(a) = B_1 (a), \, A_2(a) = B_2 (a),  A_1(a) \neq A_2 (a) \right \}.$$  

The following fact follows from the definition of $\Delta (A, B) $ and $D(-,-)$. 
\begin{fact} \label{usefulfact}  The set $\Delta (A_i , A_{i+1} )$ is the disjoint union, over all partitions of $\mc X$ into two pieces $G,H$ such that $A_i \in G$ and $A_{i+1} \in H$ of the sets $D(G,H).$ 
\end{fact} 

Now, take the sum of the inequalities \ref{firsteqn} as $i$ ranges from $1$ to $l$. On the LHS of the resulting sum, we obtain $$\sum _{i=1} ^ l \left(  \sum _{G, H \text{ a partition of $\mc X$},\, A_i \in G, A_{i+1} \in H} \left( \sum_{a \in D(G,H)}  \mu(a) u(A_i,a) \right) \right).$$
On the RHS of the resulting sum we obtain 
$$ \frac{1}{2} \sum _{i=1} ^ l \left(  \sum _{G, H \text{ a partition of $\mc X$},\, A_i \in G, A_{i+1} \in H} \left( \sum_{a \in D(G,H)}  \mu(a) \right) \right).$$
Given a partition $G,H$ of $\{ A_1, \ldots , A_l \}$ we note that the term $D(G,H) = D(H,G)$ appears exactly once as an element of the above sum for a fixed value of $i$ exactly when $A_i \in G$ and $A_{i+1} \in H$ or $A_i \in H$ and $A_{i+1} \in G.$ 

Consider the partition $G,H$ of $\mc X$. Suppose that $A_j , A_{j+1} , \ldots , A_k $ is a block of elements each contained in $G$, and that $A_{j-1}, A_{k+1} $ are in $H$. Now consider the terms 
$i=j-1$ and $i=k$ of the above sums (each of which where $D(G,H)$ appears). 

On the left hand side, we have $\sum _{a \in D(G,H) } \mu (a) u ( A_{j-1} ,a )) $ and $\sum _{a \in D(G,H) } \mu (a) u ( A_{k} ,a )) $. Note that for $a \in D(G,H)$, we have $a \in \Delta (A_{j-1}, A_k).$ So, by Lemma \ref{ulemma}, we have

$$\sum _{a \in D(G,H) } \mu (a) u ( A_{j-1} ,a )) + \sum _{a \in D(G,H) } \mu (a) u ( A_{k} ,a ))  \geq \sum _{a \in D(G,H) } \mu (a) .$$
On the RHS, we have $$\frac{1}{2} (\sum _{a \in D(G,H) } \mu (a) +\sum _{a \in D(G,H) } \mu (a)  ) = \sum _{a \in D(G,H) } \mu (a) .$$ For each $G,H$ a partition of $X$, the terms appearing in the above sum occur in pairs as above by Fact \ref{usefulfact}, and so, we have the the LHS is at least as large as the RHS of the sum of inequalities \ref{firsteqn}, which is impossible, since one of the inequalities must have been strict by our degenerate cycle. 
\end{proof}

\begin{thm} \label{qrlowbound} There is at least one element $A \in \mc C$ with query rank at least $\frac{1}{2}$. 
\end{thm} 
\begin{proof} If not, then for every element $A \in \mc C$, there is some element $B \in \mc C$ such that $d(A,B) < \frac{1}{2}$. So, pick, for each $A \in \mc C$, an element $f(A)$ such that $d(A,f(A)) < \frac{1}{2}.$ Now, fix $A \in \mc C$ and consider the sequence of elements of $\mc C$ given by $(f^i (A))$; since $\mc C$ is finite, at some point the sequence repeats itself. So, take a list of elements $B, f(B) , \ldots , f^n (B) = B$. By construction, this yields a bad cycle, contradicting Theorem \ref{nocycles}. 
\end{proof}

\subsection{The thicket max-min algorithm}\label{app:Thicketminmax}
In this subsection we show how to use the lower bound on query rank proved in Theorem \ref{qrlowbound} to give an algorithm which yields the correct concept in linearly (in the Littlestone dimension) many queries from $\mathcal{C}$. The approach is fairly straightforward---essentially the learner repeatedly queries the highest query rank concept. The approach is similar to that taken in \cite[Section 5]{angluin2017power} but with query rank in place of their notion of \emph{informative}. 

Now we informally describe the thicket max-min-algorithm. At stage $i$, the learner is given information of a concept class $\mc C_i .$ The learner picks the query 
$$A= \text{arg max} _ {A \in \mc C_i} \left( \text{min}_ {B \in \mc C_i}  \, d_{\mc C_i}( A,B) \right).$$ 
The algorithm halts if the learner has picked the actual concept $C$. If not, the teacher returns a random element $a_i \in \Delta (A,C)$ at which point the learner knows the value of $C (a_i).$ Then $$\mc C_{i+1} = (\mc C_i) _{a_i= C(a_i)}.$$ Let $T(\mc C)$ be the expected number of queries before the learner correctly identifies the target concept.

\begin{thm} The expected number of queries to learn a concept in a class $\mc C$ is less than or equal to $2 \Ldim (\mc C).$ 
\end{thm} 
\begin{proof} The expected drop in the Littlestone dimension of the concept class induced by any query before the algorithm terminates is at least $\frac{1}{2}$ by Theorem \ref{qrlowbound}; so the probability that the drop in the Littlestone dimension is positive is at least $\frac{1}{2} $ for any given query. So, from $2n$ queries, one expects at least $n$ drops in Littlestone dimension. \end{proof} 

We give a rough bound on the probability that the algorithm has not terminated after a certain number of queries. Since a query can reduce the Littlestone dimension of the induced concept class by at most $\Ldim(\mathcal{C})$ and the expected drop is at least $\frac{1}{2}$, the probability that a query reduces the Littlestone dimension is at least $\frac{1}{2\Ldim(\mathcal{C})}$. 
Then the probability that the Littlestone dimension of the induced concept class after $n$ queries is positive is at most the probability of fewer than $\Ldim(\mc C)$ many successes in the binomial distribution with probability $\frac{1}{2\Ldim(\mathcal{C})}$ and $n$ trials. It follows by Hoeffding's inequality that the probability that the algorithm has not terminated after $n$ steps is at most $$e^{-2\frac{\left(\frac{n}{2\Ldim(\mathcal{C})}-\Ldim(\mc C)\right)^2}{n}}.$$

\section{Compression schemes and stability} \label{compression}

In this section, we follow the notation and definitions given in \cite{guingonanip} on \emph{compression schemes}, a notion due to Littlestone and Warmuth \cite{littlestone1986relating}. Roughly speaking, $\mc C$ admits a \emph{$d$-dimensional compression scheme} if, given any finite subset $F$ of $X$ and some $f\in \mc C$, there is a way of encoding the set $F$ with only $d$-many elements of $F$ in such a way that $F$ \emph{can be recovered}. We will give a formal definition, but we note that numerous variants of this idea appear throughout the literature. 
For instance:

\begin{itemize} 

\item Size $d$-array compression \cite{ben1998combinatorial}. 

\item Extended compression schemes with $b$ extra bits \cite{floyd1995sample}.  

\end{itemize}

The next definition, which is the notion of compression we will work with in this section is equivalent to the notion of a $d$-compression with $b$ extra bits (of Floyd and Warmuth) \cite[see Proposition 2.1]{johnson2010compression}. 

\begin{defn} 

We say that a concept class $\mc C$ has an \emph{$d$-compression} if there is a compression function $\kappa : \mc C_{fin} \rightarrow X^{ d}$ and a finite set $\mc R$ of reconstruction functions $\rho : X^d  \rightarrow 2^X$ such that for any $f \in \mc C_{fin}$	

\begin{enumerate} 

\item $\kappa (f) \subseteq dom(f)$

\item 	$f = \rho (\kappa (f))|_{dom(f)} $ for at least one $\rho \in \mc R.$

\end{enumerate}

\end{defn}

We work with the above notion mainly because  it is the notion used in \cite{guingonanip}, and our goal is to improve a result of Laskowski appearing there \cite[Theorem 4.1.3]{guingonanip}. In \cite{johnson2010compression}, Laskowski and Johnson prove that the concept class corresponding to a stable formula has an extended $d$-compression for some $d$. The precise value of $d$ is not determined, but was conjectured to be the Littlestone dimension. A later unpublished result of Laskowski appearing as \cite[Theorem 4.1.3]{guingonanip} in fact showed that 
one could take $d$ equal to the Shelah 2-rank (Littlestone dimension) and uses $2^d$ many reconstruction functions. In Theorem \ref{JLbound}, we will show that $d+1$ many reconstruction functions suffice. 

The question of Johnson and Laskowski is the analogue (for Littlestone dimension) of a well-known open question from VC-theory \cite{floyd1995sample}: is there a bound $A(d)$ linear in $d$ such that every class of VC-dimension $d$ has a compression scheme of size at most $A(d)$? In general there is known to be bound which is at most exponential in $d$ \cite{moran2016sample}.  


\begin{defn}
	Suppose $\Ldim(\mathcal{C}) = d$. Given a partial function $f$, say that $f$ is \emph{exceptional} for $\mathcal{C}$ if for all $a \in \dom(f)$,
	\[
		\mathcal{C}_{(a, f(a))} := \{ g \in \mathcal{C} \, | \, g(a) = f(a) \}
	\]
	has Littlestone dimension $d$.
\end{defn}

%

\begin{defn}
	Suppose $\Ldim(\mathcal{C}) = d$. Let $f_{\mathcal{C}}$ be the partial function given by
	\[
		f_{\mathcal{C}}(x) = \begin{cases}
			0 & \Ldim(\mathcal{C}_{(x,0)}) = d \\
			1 & \Ldim(\mathcal{C}_{(x,1)}) = d \\
			\text{undefined} & \text{otherwise.}
		\end{cases}
	\]
\end{defn}

It is clear that $f_{\mathcal{C}}$ extends any partial function exceptional for $\mathcal{C}$.

\begin{thm} \label{JLbound}
Any concept class $\mc C$ of Littlestone dimension $d$ has an extended $d$-compression with $(d+1)$-many reconstruction functions. 
\end{thm}
\begin{proof} 
If $d = 0$, then $\mathcal{C}$ is a singleton, and one reconstruction function suffices. So we may assume $d \geq 1$.	
	
Fix some $f \in \mc C_{fin}$ with domain $F$. 
We will run an algorithm to construct a tuple of length at most $d$ from $F$ by adding one element at each step of the algorithm. During each step of the algorithm, we also have a concept class $\mc C_i$, with $\mc C_0 = \mc C$ initially. 

If $f$ is exceptional in $\mc C_{i-1}$, then the algorithm halts. Otherwise, pick either: 
\begin{itemize}

\item  $a_i \in F$ such that $f(a_i)=1$ and 
\[
	(\mc C_{i-1} )_{(a_i,1)} := \{g \,| \, g \in \mc C_{i-1}, \, g(a_i)=1 \}
\] 
has Littlestone dimension less than $\Ldim(\mathcal{C}_{i-1})$. In this case, set $\mc C_i:= (\mc C_{i-1} )_{(a_i,1)}= \{g\,| \, g \in \mc C_{i-1}, \, g(a_i)=1 \}.$ 
 
\item $d_i \in F$ such that $f(d_i)=0$ and 
\[
	(\mc C _{i-1})_{(d_i,0)} := \{g \,| \, g \in \mc C_{i-1}, \, g(d_i)=0 \}
\] 
has Littlestone dimension less than $\Ldim(\mathcal{C}_{i-1})$. In this case, set $\mc C_i:= (\mc C_{i-1})_{(d_i,0)}.$ 

\end{itemize} 

We allow the algorithm to run for at most $d$ steps. There are two distinct cases. If our algorithm has run for $d$ steps, let $\kappa(f)$ be the tuple $(\bar a, \bar d)$ of all of the elements $a_i$ as above followed by all of the elements $d_i$ as above for $i=1, \ldots , d$. By choice of $a_i$ and $d_i$, this tuple consists of $d$ distinct elements. By construction the set 
\[
	\mc C_{(\bar a , \bar d)} := \{ g \in \mc C | \, g(a_i ) = 1, \, g(d_i) =0 \}
\] 
has Littlestone dimension $0$, that is, there is a unique concept in this class. So, given $(c_1, c_2, \ldots , c_n) \in X^d$ consisting of distinct elements, for $i=0, \ldots , d$,  we let $\rho_i(c_1, \ldots, c_n)$ be some $g$ belonging to
\[
	\{g \in \mc C \, | \, g(c_j)=1 \text{ for } j\leq i, \, g(c_j )=0 \text{ for } j>i \},
\]
if such a $g$ exists. By construction, for some $i$, the Littlestone dimension of the concept class $\{g \in \mc C \cap F \, | \, g(c_j)=1 \text{ for } j\leq i, \, g(c_j )=0  \text{ for } j>i \}$ is zero, and so $g$ is uniquely specified and will extend $f$. 

We handle cases where the algorithm halts early by augmenting two of the reconstruction functions $\rho_0$ and $\rho_1$ defined above. Because $\rho_0$ and $\rho_1$ have so far only been defined for tuples consisting of $d$ distinct elements, we can extend these to handle exceptional cases by generating tuples with duplicate elements. 

If the algorithm stops at some step $i>1$, then it has generated a tuple of length $i-1$ consisting of some elements $a_j$ and some elements $d_k$. Let $\bar a$ consist of the elements $a_j$ chosen during the algorithm, and let $\bar d$ consist of the elements $d_k$ chosen during the running of the algorithm. Observe that $f$ is exceptional for $\mathcal{C}_{(\bar{a}, \bar{d})}$. 

If $\bar{a}$ is not empty, with initial element $a'$, then let $\kappa(f) = (\bar a, a' , \bar d, a', \ldots , a') \in F^d$. From this tuple, one can recover $(\bar{a}, \bar{d})$ (assuming $\bar{a}$ is nonempty), so we let $\rho_1(\bar a, a' , \bar d, a', \ldots , a')$ be some total function extending $f_{\mathcal{C}_{(\bar{a}, \bar{d})}}$, which itself extends $f$. So $\rho_1(\bar{a}, \bar{d})$ extends $f$ whenever the algorithm halts before step $d$ is completed \emph{and} some $a_i$ was chosen at some point. If $\bar{a}$ is empty, then let $\kappa(f) = (\bar d , d', \ldots , d') \in F^d$, where $d'$ is the initial element of $\bar{d}$. From this tuple, one can recover $(\emptyset, \bar{d})$ (assuming $\bar{a}$ is empty), so we let $\rho_0(\bar d , d', \ldots , d')$ be total function extending $f_{\mathcal{C}_{(\emptyset, \bar{d})}}$, which itself extends $f$. Finally, if the algorithm terminates during step 1, then it has generated the empty tuple. In this case, let $\kappa(f) = (c, \ldots, c)$ for some $c \in F$. Then $\Ldim(\mathcal{C}) = \Ldim({\mathcal{C}}_{(c, l)})$ for some $l \in \{0,1\}$. In particular, if we have defined $\kappa(f') = (c, \ldots, c)$ above for some $f'$ where the algorithm only returns $c$ (rather than the empty tuple), then $1 - l = f'(c) \neq f(c)$, and so any such $f'$ is handled by $\rho_{1-l}$. So we may overwrite $\rho_l$ to set $\rho(c, \ldots, c)$ to be a total function extending $f_\mathcal{C}$, which itself extends $f$. For any tuple output by our algorithm, one of the reconstruction functions produces an extension of the original concept. 

\end{proof}


\bibliography{Research}
\bibliographystyle{plain}

\end{document}